\documentclass[conference]{IEEEtran}
\usepackage{times}

\usepackage[numbers]{natbib}
\usepackage{multicol}
\usepackage[bookmarks=true]{hyperref}
\usepackage{amsmath}
\usepackage{amsfonts}
\usepackage{amssymb}
\usepackage{graphicx}
\usepackage{booktabs}
\usepackage{url}

\newtheorem{theorem}{Theorem}

\newtheorem{proposition}{Proposition}

\newtheorem{definition}{Definition}
\newtheorem{assumption}{Assumption}

\usepackage{algorithmicx}
\usepackage{algpseudocode}

\usepackage[linesnumbered, ruled]{algorithm2e}
\SetKwRepeat{Compute}{compute:}{result: ${\km}_{,t} = (1-\lambda) {\km}_{,t-1} + \lambda \km$}%

\newcommand{\Disc}{\text{Disc}_G(K)}
\newcommand{\atan}{\tan^{-1}}
\newcommand{\km}{\kappa_{\max}}
\newcommand{\am}{a_{\max}}
\newcommand{\dm}{d_{\max}}
\newcommand{\Rn}{\mathbb{R}^n}

\begin{document}

\title{Safe Motion Planning for Autonomous Driving using an Adversarial Road Model}


\author{\authorblockN{Alexander Liniger}
\authorblockA{CVL, ETH Zurich, Switzerland\\
Email: alex.liniger@vision.ee.ethz.ch}
\and
\authorblockN{Luc van Gool}
\authorblockA{CVL, ETH Zurich, Switzerland\\
PSI, Katholieke Universiteit Leuven, Belgium\\
Email: vangool@vision.ee.ethz.ch}
}

\maketitle

\begin{abstract}
This paper presents a game-theoretic path-following formulation where the opponent is an adversary road model. This formulation allows us to compute safe sets using tools from viability theory, that can be used as terminal constraints in an optimization-based motion planner. Based on the adversary road model, we first derive an analytical discriminating domain, which even allows guaranteeing safety in the case when steering rate constraints are considered. Second, we compute the discriminating kernel and show that the output of the gridding based algorithm can be accurately approximated by a fully connected neural network, which can again be used as a terminal constraint. Finally, we show that by using our proposed safe sets, an optimization-based motion planner can successfully drive on city and country roads with prediction horizons too short for other baselines to complete the task.
\end{abstract}

\IEEEpeerreviewmaketitle

\section{Introduction} \label{sec:intro}
Over the last few decades, self-driving cars have improved dramatically, from the beginnings in the late 80's \cite{dickmanns1987autonomous}, to the DARPA challenges \cite{Buehler2007,Buehler2009}, to the autonomous cars of today, that can autonomously navigate public roads. 

The drastic advances were achieved by improving the sensors but also the algorithms within the software stack of an autonomous car. The software stack of an autonomous car can be roughly split into four layers: localization, perception, planning, and control. In this paper, we focus on the planning and control layer, and more specifically on the motion planning problem, see \cite{paden2016} for an overview of motion planning algorithms. To goal of this paper is to derive safe sets for motion planning that can be used as a terminal constraint in a Model Predictive Controller (MPC). MPC is a popular approach for motion planning \cite{Borrelli2005,Gao2010,DiCairano2010,Liniger_2014} since it allows to directly include temporal and spatial collision avoidance constraints, as well as comfort constraints. Additionally, it is straight forward to plan a motion \emph{without} separating longitudinal and lateral dynamics, which allows for a natural motion planning formulation. However, MPC comes with several drawbacks. In this paper, we focus on the issue of recursive feasibility, where the issue is that even though constraints are considered within the prediction horizon, it is not guaranteed that the constraints can be respected when the controller is run in a receding horizon fashion since the sequence of optimization problems becomes coupled. Recursive feasibility can be achieved by imposing a terminal set constraint \cite{Mayne2000}. However, even though the theory exists for nonlinear MPC \cite{Kerrigan2000b}, it is complex to compute terminal sets for this class of problems. Therefore, in practice, one common approach to deal with this issue is to use long prediction horizons, which in terms leads to prohibitively long computation times. 

In this paper, we specifically focus on the problem of following a path with safety guarantees. Formal and provable safety guarantees for planning and decision making in autonomous driving are currently studied a lot, with one good example being the Responsibility-Sensitive Safety (RRS) approach \cite{rss}. However, these approaches mainly deal with the other road users, while guaranteed following of a path is surprisingly studied less. This is mainly addressed in the subfield of autonomous racing \cite{Liniger_2017,rosolia_2017}, where the fact that the race track is known in advance and does not change is leveraged. The same approach, however, cannot be applied directly to autonomous driving, as it would require computing the safe set for the entire road network. Therefore, there only exist a small number of papers that tackle safety and recursive feasibility for motion planning. For example in \cite{Althoff2009,Wood2012,Nilsson2014,Liniger_2017,Berntorp_2017,schouwenaars2004receding} viability, reachability, and invariant set analysis have been used to guarantee safety. In \cite{rosolia_2017} recursive feasibility is guaranteed using past data and the repetitive nature of autonomous racing. Another approach do introduce safety is by considering uncertainty and tracking errors when generating the trajectory \cite{majumdar2017funnel,herbert2017fastrack,singh2017robust,smith2019continuous}. These methods often use tools from game theory similar to our method, but our safety guarantees are for a path (the road) which is fixed.

In this work, we solve this problem by treating the road ahead as an adversarial player. This allows us to use tools from viability theory and game theory \cite{Aubin2009,Cardaliaguet99}. Given these tools, we design a motion planner that is guaranteed to follow a path which is not known in advance, but has a bounded curvature, all of that while guaranteeing that the car is staying within road limits and fulfilling comfort constraints.

The contributions of this paper are the following. First, we formulate the path-following problem as a game between the controller that wants to follow the road and an adversarial road, that tries to get the car off the road. Based on this game we find an analytical expression for a discriminating domain of the game, which can be used as a terminal set that proves recursive feasibility. Second, we compute the discriminating kernel using a gridding based algorithm and show that it can be accurately approximated using a fully connected neural network. This neural network can again be used as a terminal constraint for our path-following MPC. Finally, we evaluate our proposed terminal constraints together with a path-following MPC in a simulation study and show that they outperform common baselines, and work well with short prediction horizons where the baselines fail. 

This paper is organized as follows. Section \ref{sec:MPC} introduces the path-following MPC problem, and Section \ref{sec:viab} summarizes the discriminating kernel algorithm of \cite{Cardaliaguet99}. In Section \ref{sec:safeset} the path-following task is reformulated as a game and the discriminating domain and kernel for the task are described and computed. In Section \ref{sec:sim} we study the performance of the controller in simulation and compare it to several baseline terminal constraints. Finally, in Section \ref{sec:con} we conclude the paper and give an outlook.

\section{Path-Following MPC}\label{sec:MPC}

In this section, we introduce our path-following MPC problem including all the necessary components, such as the model, constraints, and cost function. Note that our safe sets, derived in Section \ref{sec:safeset}, are also based on the modeling and the constraints introduced here.

\subsection{Model}
In this paper, we use a kinematic bicycle model formulated in curvilinear coordinates. Thus the position is not formulated in global coordinates, but only the local deviations with respect to a known path are considered, see \cite{rajamani2011}. More precisely, in the curvilinear coordinate system we only consider the progress along the path $s$, the orthogonal distance to the path $d$, as well as the local heading compared to the path $\mu$, see Figure \ref{fig:curvi_model} for an illustration and \cite{rucco_2015} for more details. Note that given the path, the global coordinates can be recovered, and given the global coordinates the curvilinear coordinates can be obtained, given some assumptions.

To formulate the dynamics of our MPC, we consider the kinematic bicycle conditions (see Figure \ref{fig:curvi_model}), use the steering angle $\delta$ as an input and augment the system with a dynamic velocity state $v$, that is the integration of the acceleration input $a$. Finally, we consider the influence of the curvature $\kappa$ on the evolution of the curvilinear coordinates, which results in the following dynamics,
\begin{align}
    \dot{s} &= \frac{v \cos(\mu)}{1 - d \kappa(s)} \,, \nonumber\\
	\dot{d} &= v \sin(\mu) \,, \nonumber\\
	\dot{\mu} &= \frac{v \tan(\delta)}{L} - \kappa(s) \frac{v \cos(\mu)}{1 - d \kappa(s)}\,, \nonumber\\
	\dot{v} &= a\,. \label{eq:dyn_mpc}
\end{align}
Where, the state and input are given by $x = (s,d,\mu,v)$, and $u = (\delta,a)$ respectively, and $\kappa(s)$ is the curvature at the progress $s$. We denote the resulting continuous time dynamics \eqref{eq:dyn_mpc} as $\dot{x} = f_c(x,u)$, and the discrete time version obtained by discretizing the dynamics using an Ordinary Differential Equation (ODE) integrator as $x_{k+1} = f(x_k,u_k)$. Note that the rest of the MPC ingredients are all derived in discrete time, since we use a standard discrete time MPC formulation.

\begin{figure}[h]
\centering
\vspace{0.1cm}
\includegraphics[width = 0.35\textwidth]{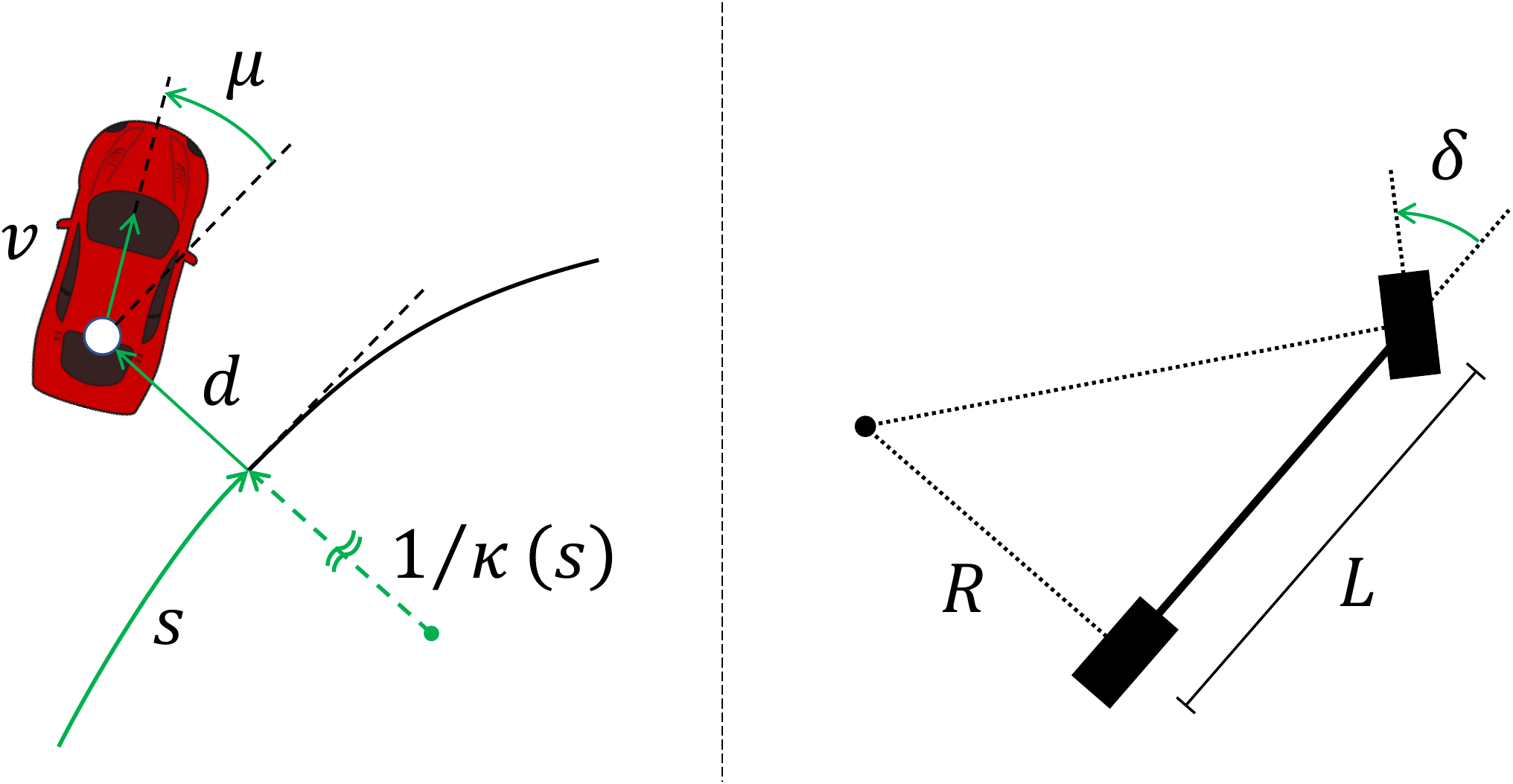}
\caption{Right: Visualization of the curvilinear coordinates. Left: Visualization of the kinematic bicycle conditions.} \label{fig:curvi_model}
\end{figure}
Note that the progress state $s$, is only linked to the other states due to the $s$-dependent curvature. However, $s$ could also be used to implement collision constraints.

\subsection{Constraints}

In our MPC formulation, we consider two types of constraints: acceleration constraints and road constraints. The acceleration constraint is fundamental for two reasons: first, it guarantees comfort for the people inside the car, and second, the kinematic model \eqref{eq:dyn_mpc} does not consider slip in the tires and would overestimate the velocity in corners in the absence of the acceleration constraint.

The acceleration constraint is formulated by limiting the combined acceleration in longitudinal and lateral direction to be smaller than $\am$. For the model in \eqref{eq:dyn_mpc}, the longitudinal acceleration is given by $a_\text{long} = a$, and the lateral acceleration is $a_\text{lat} = v^2 \tan(\delta)/L$. The lateral acceleration is given by the centripetal acceleration of the car, combined with the fact that by using a kinematic car model we know that the car is driving a circle with a radius of $R = L/\tan(\delta)$, see Figure \ref{fig:curvi_model}.

Thus, the acceleration constraint can be formulated as a velocity-dependent input constraint,
\begin{align}
\mathcal{U}(v) = \{\delta \in \mathcal{D}, a \in \mathcal{A} \, | \, (v^2 \tan(\delta) / L)^2 + a^2 \leq \am^2\}, \label{eq:acc_con}
\end{align}
where $\mathcal{D} = [-\bar{\delta},\bar{\delta}]$ and $\mathcal{A} = [\underline{a},\bar{a}]$ are the physical limits of the inputs.

The road constraints force the car to remain within the lane. By using curvilinear coordinates, the road constraints can be implemented by limiting $d$; however, to guarantee that the whole car remains within the road, the local heading of the car has to be considered. To formulate the constraint, let $L_{c}$ and $W_c$ be the length, and the width of the car respectively, and $l_r$ the distance from the center of the rear axle to the geometric center of the car, see Figure \ref{fig:road_con} for an illustration. Then, the car stays within the road if,
\begin{align}
    &C_w = W_c/2 \cos(\mu) + L_c/2 \sin(|\mu|) \,, \nonumber \\
    &d + l_r \sin(\mu) + C_w -W_{r,l} \leq 0 \,, \nonumber \\
   -&d - l_r \sin(\mu) + C_w + W_{r,r}\leq 0 \,, \label{eq:road_con}
\end{align}
where $C_w$ is the width of the car relative to the lane (considering the local heading), $W_{r,l}$ the left road width and $W_{r,r}$ the right road width, both relative to the reference path (in our experiments we use the center line of the lane). Note that $W_c$ is as such not differentiable; however, it can be exactly reformulated as a smooth constraint, by using four instead of two constraints. To simplify the notation, in Section \ref{sec:safeset} we denote the road constraint \eqref{eq:road_con} as $h_r(d,\mu) \leq 0$. 

\begin{figure}[h]
\centering
\vspace{0.1cm}
\includegraphics[width = 0.35\textwidth]{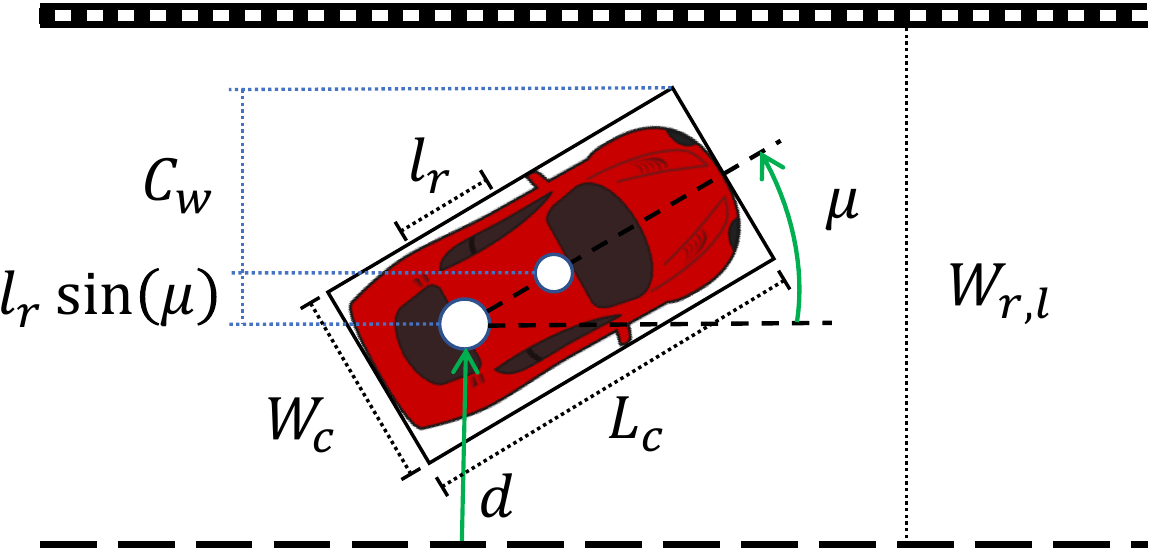}
\caption{Visualization of the road constraints.} \label{fig:road_con}
\end{figure}

Finally, we also constrain the relative heading to be within reasonable bounds $\underline{\mu} \leq \mu \leq \bar{\mu}$, and the velocity to be between zero and the speed limit $v_{\max} (s)$ at the current progress $s$. The combination of the road constraints as well as the heading and velocity constraints is denoted as $x_k \in \mathcal{X}$

\subsection{Terminal Constraint}

As discussed in the Introduction \ref{sec:intro}, for an MPC to have safety guarantees, terminal constraints are necessary. In Section \ref{sec:safeset} a solution is presented for how to tackle this problem for the here presented path-following MPC. However, for now, we assume that we have an abstract terminal constraint of the form $x_{N+1} \in \mathcal{X}_T$, where $N+1$ is the prediction horizon.

\subsection{Cost}

The cost of the MPC is the combination of path-following and comfort costs. More precisely, the path-following cost consists of two parts: $j_{pf}$ a quadratic cost on $d$ and $\mu$ that penalizes deviations from the path, and $j_p$ a maximum progress cost,
\begin{align}
    &j_{pf}(x_k) = q_d d_k^2 + q_\mu \mu_k^2\,, \nonumber\\
    &j_p(x_k) = -q_p s_{N+1}\,. \label{eq:pf_cost}
\end{align}
Where $q_d$, $q_\mu$, and $q_p$ are positive weights. Note that instead of a maximum progress cost it is also possible to penalize the deviation from the speed limit. However, we use the maximum progress cost since speed limits are often not appropriate on curvy and urban roads. 

The comfort cost consists of two terms, an acceleration cost $j_a$, that penalizes lateral and longitudinal acceleration, and an input rate cost $j_r$ that penalizes changes in the inputs or in other words jerky motions,
\begin{align}
    &j_{a}(x_k,u_k) = q_\text{lat} (v_k^2\tan(\delta_k)/L)^2 + q_\text{long} a_k^2 \,, \nonumber\\
    &j_{r}(u_k,u_{k-1}) = q_{\Delta \delta} \Delta \delta_k^2 + q_{\Delta a} \Delta a_k^2\,. \label{eq:comf_cost}
\end{align}
Where $q_\text{lat}$, $q_\text{long}$, $q_{\Delta \delta}$, and $q_{\Delta a}$ are positive weights and the input rates are defined as follows: $\Delta \delta_k := \delta_k - \delta_{k-1}$ and $\Delta a_k := a_k - a_{k-1}$.

The path-following cost \eqref{eq:pf_cost} and the comfort cost \eqref{eq:comf_cost} are combined to the following stage cost 
\begin{align*}
    \ell(x_k, u_k, u_{k-1})  = &j_{pf}(x_k) + j_p(x_k) \\
                           & + j_{a}(x_k,u_k) + j_{r}(u_k,u_{k-1})\,,
\end{align*}
where $u_{-1}$ is the known last applied input. Additionally, we also introduce a terminal cost which is a modified path-following cost \eqref{eq:pf_cost},
\begin{align*}
    \ell_T(x_{N+1}) = q_{d,T} d_{N+1}^2 + q_{\mu,T} \mu_{N+1}^2\,,
\end{align*}
where the weights $q_{d,T}$ and $q_{\mu,T}$ are increased compared to the normal weights.

\subsection{MPC Problem}

Combining all of the previous ingredients, we can formulate the path-following MPC as the following Nonlinear Optimization Problem (NLP),
\begin{equation}\label{eq:MPC_genCollAvoid}
	\begin{array}{ll}
		\displaystyle\min_{\mathbf x, \mathbf u}  & \displaystyle\sum_{k=0}^{N} \ell(x_k, u_k, u_{k-1}) +  \ell_T(x_{N+1})\\
		\ 	\text{s.t.} & x_0 = x(0), \; u_{-1} = u(-1),\\
					&x_{k+1} = f(x_k,u_k), \\
					& x_k \in \mathcal{X}, \quad x_{N+1} \in \mathcal{X}_T,\\
					& u_k \in \mathcal{U}(v_k),\\
					& k=0,\ldots,N\,.
	\end{array}
\end{equation}
Where $\mathbf x := [x_0,x_1,\ldots,x_{N+1}]$ is the collection of all states and $\mathbf u := [u_0,u_1,\ldots,u_{N}]$ is the collection of all inputs. Additionally, $x(0)$ is the measured state and $u(-1)$ is the previously applied input.

\section{Viability Theory} \label{sec:viab}
In this section, we will briefly discuss discrete-time viability theory \cite{Cardaliaguet99}, which we will later use to establish safe terminal constraints for our path-following MPC. More precisely we are interested in the game-theoretic version of viability theory. There the goal is to establish a \emph{victory domain} for which there exists an input policy that is able to keep the state of the system within a constraint set forever, for all adversarial actions \cite{Cardaliaguet99}.

Therefore, let us consider a discrete-time system with two inputs $z_{k+1} = g(z_k,u_k,\nu_k)$, where $z\in \Rn$ is the state, $u \in \mathcal U \subset \mathbb{R}^m$ is the control input, $\nu \in \mathcal V \subset \mathbb{R}^d$ is the adversarial input and $g:\Rn \times \mathcal U \times \mathcal V \rightarrow \Rn$ is a continuous function. Furthermore, the constraint set is given as $K \subset \Rn$ and is assumed to be a closed subset of $\Rn$. Furthermore, the policies have the following information pattern when computing the \emph{victory domain}: the control player has access to both the state $z_k$ as well as the adversarial input $\nu_k$ thus $u_k = \pi_u(z_k,\nu_k)$, whereas the adversarial player only has access to the state $\nu_k = \pi_\nu(z_k)$. Finally, the goal of the control player is to stay within the constraint set $K$ forever, whereas the adversarial player wants to reach the open set $\Rn \setminus K$.

Given these preliminaries let us now state the standard results from viability theory for this setting.

\subsection{Discriminating Domain and Kernel}

The difference equation $z_{k+1} = g(z_k,u_k,\nu_k)$ can be reformulated as a difference inclusion with one input,
\begin{align}
    z_{k+1} \in G(z_k,\nu_k) \quad \text{with } G(z,\nu) = \{g(z,u,\nu)|u \in U\}\,.
\end{align}
Thus, the difference inclusion predicts all the states for all allowed actions, given an adversarial input. Given the set-valued map $G(z,\nu)$ we can define a discriminating domain and the discriminating kernel.

\begin{definition}[\cite{Cardaliaguet99}]
A set $Q \subset \Rn$ is a discrete discriminating domain of $G$, if for all $z \in Q$, it holds that $G(z,\nu) \cap Q \neq \emptyset$ for all $\nu \in \mathcal V$. The \emph{discrete  discriminating  kernel} of  a  set $K \subset \Rn$ under $G$, denoted by $\Disc$, is the largest closed discrete discriminating domain contained in $K$.
\end{definition}

One can show that under the given assumptions, namely that $K$ is a closed subset of $\Rn$ and $G:\Rn \times \mathcal V \rightsquigarrow \Rn$ is an upper-semicontinuous set-valued map with compact values, the discriminating Kernel $\Disc$ exists. 

\subsection{Discriminating Kernel Algorithm}
The discriminating kernel $\Disc$ can conceptually be computed using the discriminating kernel algorithm, which generates a sequence of nested closed sets,
\begin{align}
    &K^0 = K\,,\nonumber \\
    &K^{n+1} = \{z \in K^n | \forall \nu \in \mathcal V, G(z,\nu) \cap K^n \neq \emptyset\}\,. \label{eq:disc_algo}
\end{align}
\begin{proposition}[\cite{Cardaliaguet99}]
Let $G$ be an upper-semicontinuous set-valued map with compact values and let $K$ be a closed subset of $\Rn$, then, $\bigcap_{n=0}^\infty = \Disc$
\end{proposition}

However, since the discriminating kernel algorithm has to operate on arbitrary sets, it is as such not implementable. One popular method to solve this problem is to discretize the state space as well as the input space of the adversarial player. If the state space is appropriately discretized, \cite{Cardaliaguet99} showed that the discrete space discriminating kernel converges to the discriminating kernel if the space discretization goes to zero. 

\section{Safe Set for Path-Following} \label{sec:safeset}
The main issue when trying to establish a safe set for the path-following task is that the MPC only knows the road until the end of the prediction horizon. Thus, one popular approach to solve the issue is to force the car to stop at the end of the horizon \cite{Kuwata_2009}. However, this requires a long prediction horizon, such that the terminal constraint does not negatively influence the current decisions. In this paper we propose to treat the unknown road ahead as an adversarial player, which allows us to compute the discriminating kernel, which in turn guarantees infinite horizon constraint satisfaction.

\subsection{Road Curvature as an Adversary}
In our curvilinear dynamics \eqref{eq:dyn_mpc}, the road ahead is captured by the curvature. Thus, modeling the road ahead as an adversarial player is equivalent to modeling the curvature as an input and limiting the input to stay within an upper and lower bound. This makes the curvature independent of the progress, and therefore we can neglect the progress dynamics for the computation of the discriminating kernel, as it is not linked to any of the other states or the constraints. Thus, we can introduce a reduced state $z = (d,\mu,v)$. The input remains the same $u=(\delta,a)$, but the system now has an adversarial input $\nu = \kappa$, and the system dynamics are given by,
\begin{align}
	\dot{d} &= v \sin(\mu) \,, \nonumber\\
	\dot{\mu} &= \frac{v \tan(\delta)}{L} - \kappa \frac{v \cos(\mu)}{1 - d \kappa}\,, \nonumber\\
	\dot{v} &= a\,. \label{eq:dyn_disc}
\end{align}

Similar to \eqref{eq:dyn_mpc}, we denote the continuous dynamics as $\dot{z} = g_c(z,u,\nu)$, and the discretized dynamics as $z_{k+1} = g(z_k,u_k,\nu_k)$. Similarly the input constraint set is given by \eqref{eq:acc_con}, and the state constraint set is given by,
\begin{align*}
K = \{ z \in \mathbb{R}^3 \; | \; -&\underline{\mu} \leq \mu \leq \bar{\mu}, \,, 0\leq v\leq \bar{v},\\
    &h_r(d,\mu) \leq 0\,,
\end{align*}
where $\bar{v}$ is the maximum speed limit. Finally, the action of the adversarial player is limited to $\mathcal V = [-\km,\km]$, where $\km$ is the maximal curvature of the road.

Thus, we can now characterize the discriminating kernel for our path-following system.

\subsection{Analytical Discriminating Domain}
First, we formulate an analytical discriminating domain, which by definition is a subset of the discriminating kernel. Even though it is only a subset of the discriminating kernel, the advantage is that the set is defined analytically and therefore no expensive computations are necessary. We additionally have to make the following mild assumptions.
\begin{assumption}
The set $K$ has a relative interior and there exists a $d$ such that $\mu = 0$ is feasible. \label{ass:con_set}
\end{assumption}
\begin{assumption}
The road constraints $\eqref{eq:road_con}$ and $\km$ are such the $d \kappa < 1$ is always fulfilled. \label{ass:curvi}
\end{assumption}
\begin{assumption}
For the ODE integrator that discretizes the continuous time dynamics $\eqref{eq:dyn_disc}$, it holds that $z_{k+1} = z_k$ if $g_c(z,u,\nu) = 0$ \label{ass:integrator}
\end{assumption}

The first assumption \ref{ass:con_set}, informally guarantees that the road is wide enough. Assumption \ref{ass:curvi} guarantees that the curvilinear dynamics are well defined. Finally, Assumption \ref{ass:integrator} guarantees that the ODE integrator is well behaved in steady-state, which is fulfilled for the integrators used in this paper. 

Finally, for the following results, we need $\dm$ which is the ``largest" allowed deviation from the path. For the special case when $\mu = 0$, this can simply be found by considering the road constraints \eqref{eq:road_con}, thus $\dm = \max(|W_{r,l} + W_c/2|,|W_{r,r} - W_c/2|)$. Note that $\dm$ is also the deviation $d$ for which the disturbance $\kappa$ has the largest influence on the system $\eqref{eq:dyn_disc}$.

Thus we can now state our main result which is the analytical discriminating domain.
\begin{theorem}
Given that Assumptions \ref{ass:con_set}, \ref{ass:curvi}, and \ref{ass:integrator} hold, and that $\km  \geq \tan (\bar{\delta})  / (L  + \dm \tan (\bar \delta))$, then 
\begin{align}
    \mathcal{Z}_{DD} = \begin{cases}
        &W_{r,l} + W_c/2 \leq d \leq W_{r,r} - W_c/2\\
        &\mu = 0\\
        &0 \leq v \leq \min \left(\bar{v}, \sqrt{\frac{\am (1 - |d \km|)}{\km}} \right)
    \end{cases}\,,
\end{align}
is a discriminating domain. \label{theo:disc_dom}
\end{theorem}
\begin{proof}
Since, $\mathcal{Z}_{DD} \subset K$, it is sufficient that there exists an input in $u \in U(v)$ such that $z_{k+1} = z_k$ for all $\nu \in \mathcal V$, for $\mathcal{Z}_{DD}$ to be a \emph{discriminating domain}, since in this case $G(z,\nu) \cap \mathcal{Z}_{DD} \neq \emptyset$ holds. Furthermore, $z_{k+1} = z_k$ holds, if the continuous time dynamics \eqref{eq:dyn_disc} are stationary, $g_c(z,u,\nu) = 0$. Which can be achieved by designing the following policy,
\begin{align}
    \pi_u(z,\nu) = \begin{cases}
        \delta = \atan\left(\frac{\kappa L}{1 - d \kappa} \right)\\
        a = 0
    \end{cases}\,. \label{eq:policy}
\end{align}
The policy \eqref{eq:policy} fulfills the input constraints if $0 \leq v \leq \sqrt{\am(1 - d \km)/\km}$ and if $\km  \geq \tan (\bar{\delta}) / (L  + \dm \tan (\bar \delta)$. The curvature limitation is needed for the physical limit of the steering input, whereas the velocity limitation comes form the lateral acceleration limits.

First, using $a=0$ guarantees that the input constraints are fulfilled and allows to simplify the input constraints for $\delta$ to $-\atan (\am L/v^2)  \leq \delta  \leq \atan(\am L/v^2)$. Plugging in the policy, and only considering the upper bound due to the symmetry of $\mathcal U(v)$ and $\mathcal V$, results in $\atan (\kappa L/(1 - d \kappa)) \leq \atan(\am L/v^2)$, which holds true for all $\kappa \in [-\km,\km]$ if $0 \leq v \leq \sqrt{\am(1 - |d \km|)/\km}$

If $\km  \geq \tan (\bar{\delta}) / (L  + \dm \tan (\bar \delta))$ then it holds that $\delta  = \atan (\kappa L/(1 - d \kappa) ) \leq \bar{\delta}$ for all $z \in \mathcal{Z}_{DD}$. This is simple to verify by reformulating, $\atan(\km L / (1 - \dm \km)) \leq \bar{\delta}$, as this guarantees that the biggest steering angle of the policy is within the bounds.
\end{proof}

Due to the analytical nature of the discriminating domain established in Theorem \ref{theo:disc_dom}, it is possible to generalize Theorem \ref{theo:disc_dom} to the interesting case of steering rate constraints. This can be done if the road curvature changes slowly enough compared to the steering rate constraint. 

\begin{proposition}
If the rate of change of the steering input is limited by $|(\delta_k - \delta_{k+1})| \leq \overline{\Delta \delta}$ and the rate of change of the curvature is limited by $|(\kappa_k - \kappa_{k+1})| \leq \tan(\overline{\Delta \delta})(1-\dm \km)/L$ then $\mathcal{Z}_{DD}$ is a discriminating domain.
\end{proposition}
\begin{proof}
Form the proof of Theorem \ref{theo:disc_dom}, we know that the policy \eqref{eq:policy}, renders $\mathcal{Z}_{DD}$ invariant, while not violating $U(v)$. 

To investigate the case where the rate of change is limited, we can describe the curvature at time $k+1$ as $\kappa_{k+1} = \kappa_k + \Delta \kappa$, with $|\Delta \kappa| \leq \tan(\overline{\Delta \delta})(1-\dm \km)/L$. Thus the steering angle is given by $\delta_{k+1} = \atan((\kappa_k + \Delta \kappa)L/(1 - d \kappa))$. Since we are only interested in the absolute value of the difference we can take the absolute value of the expression, which allows us to make a series of simplifications,
\begin{align*}
|\delta_{k+1}|  &= |\atan((\kappa_k + \Delta \kappa)L/(1 - d \kappa))|\\
                &\leq |\atan(\kappa_k L/(1 - d \kappa))| + |\atan( \Delta \kappa L/(1 - d \kappa))|\\
                &\leq |\delta_k| + |\atan( \Delta \kappa L/(1 - d \kappa))|\\
                &\leq |\delta_k| + |\atan( \Delta \kappa L/(1 - \dm \km))|
\end{align*}
Which can be reformulated as, $|\delta_{k+1}| - |\delta_k| \leq |\atan( \Delta \kappa L/(1 - \dm \km))|$.
Thus, the change in the steering angle is bounded by $|\atan( \Delta \kappa L/(1 - \dm \km))|$, which is at most $\overline{\Delta \delta}$ if $|\Delta \kappa| \leq \tan(\overline{\Delta \delta})(1-\dm \km)/L$.
\end{proof}

\subsection{Numerical Discriminating Kernel}
The discriminating domain $\mathcal{Z}_{DD}$ is only a subset of the discriminating kernel $\Disc$. However, to numerically compute the discriminating kernel using the discriminating kernel algorithm, it is necessary to fix the model and constraint set $K$. Therefore, in Table \ref{tab:param} the model and road parameters that are used for the rest of our paper are introduced. Note that the parameters are partially taken from \cite{rajamani2011} and from a spec sheet of a 2005 BMW 320i.

\begin{table}[h]
\caption{Model and Road Parameters}
\label{tab:param}
\centering 
\begin{tabular}{@{}l l @{}}\toprule
$L = 2.68$m & $l_r = 1.34$m\\
$L_c = 4.52$m & $W_c = 1.817$m\\
$\am = 1.6$ms$^{-2}$ & $\mathcal{D} = [-0.6,0.6]$rad\\
$\mathcal{A} = [-1.6,1.6]$ms$^{-2}$ & $\bar{\mu} = -\underline{\mu} = 0.2$rad\\
$W_{r,l} = -1.5$m & $W_{r,r} = 1.5$m\\
\bottomrule
\end{tabular}
\end{table} 

Additionally, it is necessary to grid the state and adversarial input space. In our implementation we also grid the input space, since this makes the implementation considerably simpler. The parameters used for the gridding can be found in Table \ref{tab:grid}, where $\bar{\delta}(v) = \min(\atan(\am L/v^2),\bar{\delta})$. Note that we use the results of Theorem \ref{theo:disc_dom} to upper bound the velocity state. We also experimented with denser grids for the inputs, and found that it did not change the resulting discriminating kernel noticeably, if at all. We assume that this is the case since the extreme inputs are the ones which are most important. 

\begin{table}[h]
\caption{Grid Parameters}
\label{tab:grid}
\centering 
\begin{tabular}{@{}l l l l l l @{}}\toprule
state & range & $\#$grid &input & range & $\#$grid\\
\midrule
$d$ & $[-0.3415,0.3415]$ & 101 &$\nu$ & $[-\km,\km]$ & 5\\
$\mu$ & $[-0.2,0.2]$ & 81 &$\delta$ & $[-\bar{\delta}(v),\bar{\delta}(v)]$ & 9\\
$v$ & $[0,\sqrt{\am/\km}]$ & 135& $a$ & $[-\am,\am]$ & 9\\
\bottomrule
\end{tabular}
\end{table} 

Given all the parameters, the discriminating kernel algorithm $\eqref{eq:disc_algo}$ is implemented in C++ and the the discriminating kernel is computed for different $\km$. We use a 4th order Runge-Kutta integrator to discretize the dynamics with a sampling time of $T_s  = 0.2$s. The sampling time is chosen to match our state discretization, which is important for discriminating kernel algorithms \cite{Cardaliaguet99}. Given the state discretization from Table \ref{tab:grid} combined with our C++ implementation the computation times were between 15.3 and 135.9s per maximal curvature. We ran 13 different maximum curvatures ${\km}_\text{, all} = \{$0.1, 0.05, 0.04, 0.03, 0.02, 0.01, 0.005, 0.004, 0.003, 0.002, 0.0015, 0.00125, 0.001$\}$m$^{-1}$, which corresponds to roads with maximum radius between 10m and 1000m, thus capturing everything from highways to roundabouts.

\subsubsection{Neural Network Approximation}
As such the discriminating kernel is not usable as a terminal constraint in an MPC since the set is only known in the discretized space, and for discrete values of $\km$. Therefore, we propose to learn a neural network classifier that approximates the discriminating kernel in continuous space and for continuous values of $\km$. We therefore train a neural network $h_{nn}:\mathbb{R}^3 \times \mathbb{R} \rightarrow [0,1]$, that maps states and curvatures to one output. The constraint is then implemented as $h_{nn}(x,\km) \geq c$, where $c \in (0,1)$.

We use Fully Connected Neural Networks (FCNN) since they can easily handle the massive amount of data that the gridding based discriminating kernel algorithm generates. For our 13 different $\km$, we already have about 14.5 million data points, while the amount of data points grows exponentially when finer grids are used. Furthermore, FCNN can be included in an MPC formulation, since they are at least once continuous differentiable nonlinear functions if appropriate activation functions are used.

We perform a small ablation study to compare different network architectures, where we mainly compare activation functions as well as the number of neurons and layers of the network. Our study is set up as a sensitivity analysis, where we compare everything to our base case, which is a three-layer network with 16 neurons per layer, using ELU activation functions. For all our networks we use a sigmoid output layer to map the output to $[0,1]$, and use a binary cross-entropy as a training loss. The training is consistent for all the different networks, and we train the networks for 9 epochs using ADAM as an optimizer, with an initial learning rate of 0.01, which is reduced every 3 epochs by one order. Finally, we use a batch size of 1500 which allows us to train our networks within a few minutes on an i7 desktop processor. 

\begin{table}[t]
\caption{Activation Functions, 3 layers and 16 neurons, (test/val) }
\label{tab:activations}
\centering 
\begin{tabular}{@{}l c c c@{}}\toprule
activation & acc [\%]  & fn[\%] & fp[\%]\\
\midrule
\textbf{ELU} & \textbf{99.33/99.19} & \textbf{0.61/0.70} & \textbf{0.05/0.11} \\
ReLU & 98.82/99.02 & 0.84/0.88 & 0.34/0.11 \\
Softplus & 99.23/99.15 & 0.71/0.74 & 0.05/0.11 \\
Tanh & 99.30/99.22 & 0.65/0.66 & 0.04/0.11 \\
\bottomrule
\end{tabular}
\end{table}

As training dataset we use the union of the 13 computed discriminating kernels. We label points inside the discriminating kernel as 1 (safe) and points outside as 0 (unsafe). Finally, we normalize the dataset and randomly split it between a training (95\%) and a validation set (5\%). Our test set is designed to test how well the neural network generalizes to unseen $\km$. Therefore, we use two additional discriminating kernels with curvatures ${\km}_\text{, test}  = \{0.015, 0.0035\}$m$^{-1}$. The test kernels use a finer grid with twice the grid points in each direction.

To compare the results we use the following metrics: accuracy (acc), false-positive (fp), and false-negative (fn) rates. We on purpose choose the cut-off $c$ of the neural network output to result in low false-positive rates, as otherwise the states close to the top forward velocity are often deemed unsafe; specifically we chose $c = 0.25$. Note that the network that we use in the rest of the paper is always highlighted bold in the ablation study tables.

\begin{table}[t]
\caption{Number of Neurons, ELU activation and 3 layers,  (test/val)}
\label{tab:neurons}
\centering 
\begin{tabular}{@{}l c c c@{}}\toprule
\#neurons & acc [\%] & fn[\%] & fp[\%]\\
\midrule
8 & 98.96/98.88 &  0.94/0.96 & 0.09/0.15 \\
\textbf{16} & \textbf{99.33/99.19} & \textbf{0.61/0.70} & \textbf{0.05/0.11} \\
32 & 99.42/99.36 &  0.53/0.53 & 0.05/0.1 \\
128 & 99.52/99.53 &  0.42/0.39 & 0.05/0.09 \\
\bottomrule
\end{tabular}
\end{table} 
            
\begin{table}[t]
\caption{Number of Layers, ELU and 16 neurons,  (test/val)}
\label{tab:layers}
\centering 
\begin{tabular}{@{}l c c c@{}}\toprule
\#layers & acc [\%]  & fn[\%] & fp[\%]\\
\midrule
2 & 98.89/98.73  & 1.0/1.14 & 0.11/0.12 \\
\textbf{3} & \textbf{99.33/99.19} & \textbf{0.61/0.70} & \textbf{0.05/0.11} \\
6 & 99.41/99.40 & 0.51/0.52 & 0.08/0.08 \\
\bottomrule
\end{tabular}
\end{table} 

Our sensitivity analysis mainly reassures common knowledge: deeper or wider networks get better scores, see Table \ref{tab:layers} and \ref{tab:neurons}, which indirectly shows that over-fitting seems not to be an issue for this task. The only somewhat unexpected outcome is that ReLU activation functions perform worst for the given task, which can be seen in Table \ref{tab:activations}. We did not investigate this behavior further since ReLUs are not suited to be included in an MPC problem as they are not continuously differentiable. It is also interesting that in general the networks performed better on the test data, than on the validation data. This could be related to the fact that the test data only contains two unseen curvatures. However, the testset still shows that the neural networks learned to successfully interpolate in the curvature dimension. 

Our final decision to go with the selected network, was based on computational considerations, as wider and deeper networks increase the computational burden during of the MPC, and Table \ref{tab:layers} and \ref{tab:neurons} show that three layers and 16 neurons strike a good balance between accuracy and simplicity\footnote{Code for the discriminating kernel algorithm and FCNN learning available at \url{https://github.com/alexliniger/AdversarialRoadModel}}. Furthermore, the ELU activation functions had accuracy scores only rivaled by $\tanh$ activation functions, with the approximation having visually less artifacts\footnote{We tested different random seeds for the ELU network and got accuracy scores between 99.29 - 99.34\% on the testset. We report scores for one of the intermediate runs with 99.33\% accuracy.}. For a visualization of the resulting neural network based discriminating kernel and the comparison to the result of the discriminating kernel algorithm see Figure \ref{fig:NNDiscKernel}, which shows the test case with $\km = 0.0035$m$^{-1}$.

\begin{figure}[h]
\centering
\includegraphics[width = 0.375\textwidth]{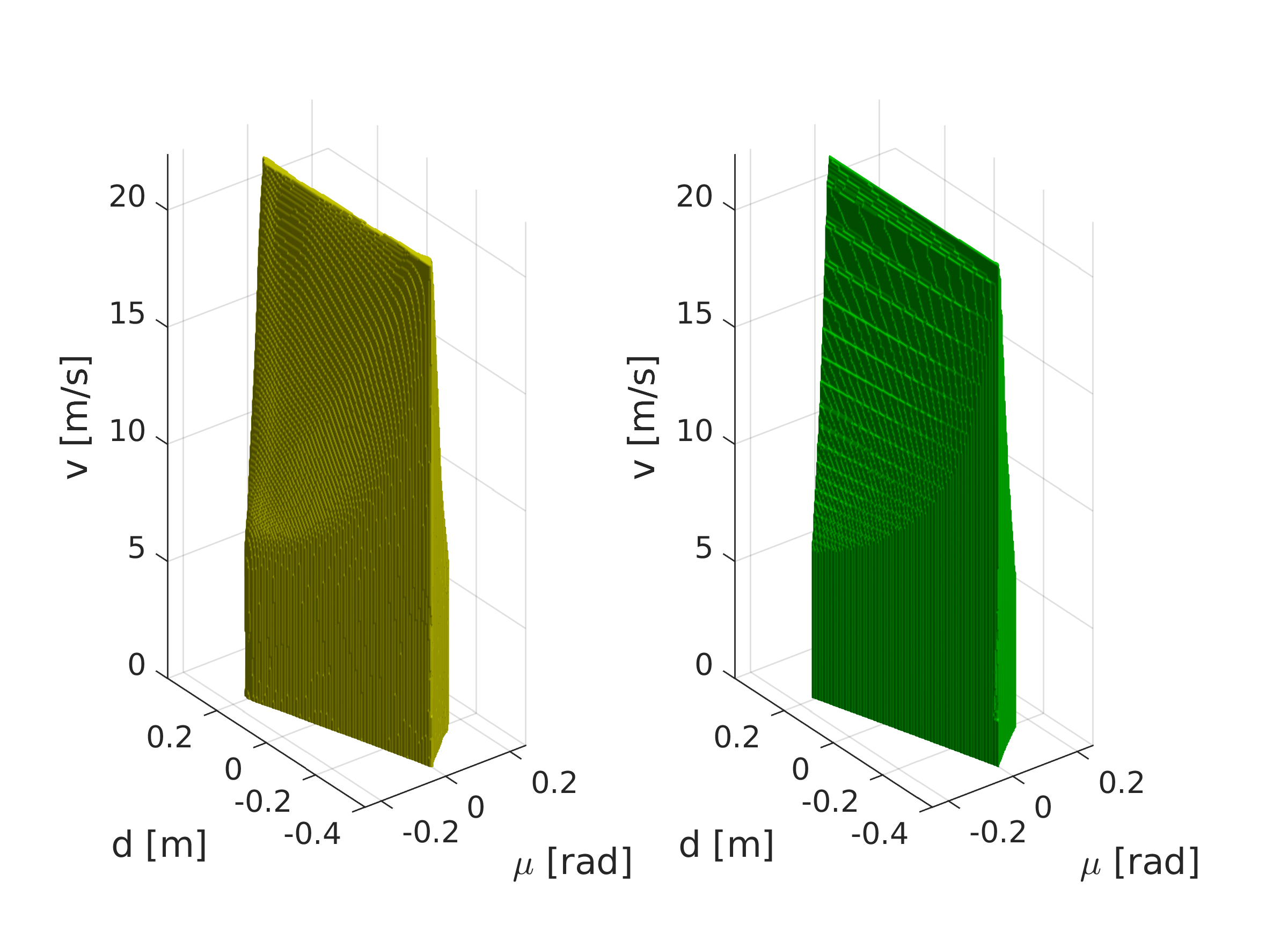}
\caption{Visualization of the discriminating kernel for $\km = 0.0035$m$^{-1}$, on the left the neural network approximation and on the right the ground truth from the discriminating kernel algorithm.} \label{fig:NNDiscKernel}
\end{figure}

\section{Simulation} \label{sec:sim}
To test the proposed terminal constraints, we performed two tests, first driving along a curvy country road, where the first part is inside a town with a 50km/h speed limit and the second part has a 80km/h speed limit, but a tight turn. The second test should show the performance in a city, and consist of two 90$^\circ$ curves, followed by a 45$^\circ$ curve. The curvature for the two test cases can be seen in Figure \ref{fig:test_curvature}.

\begin{figure}[h]
\centering
\includegraphics[width = 0.35\textwidth]{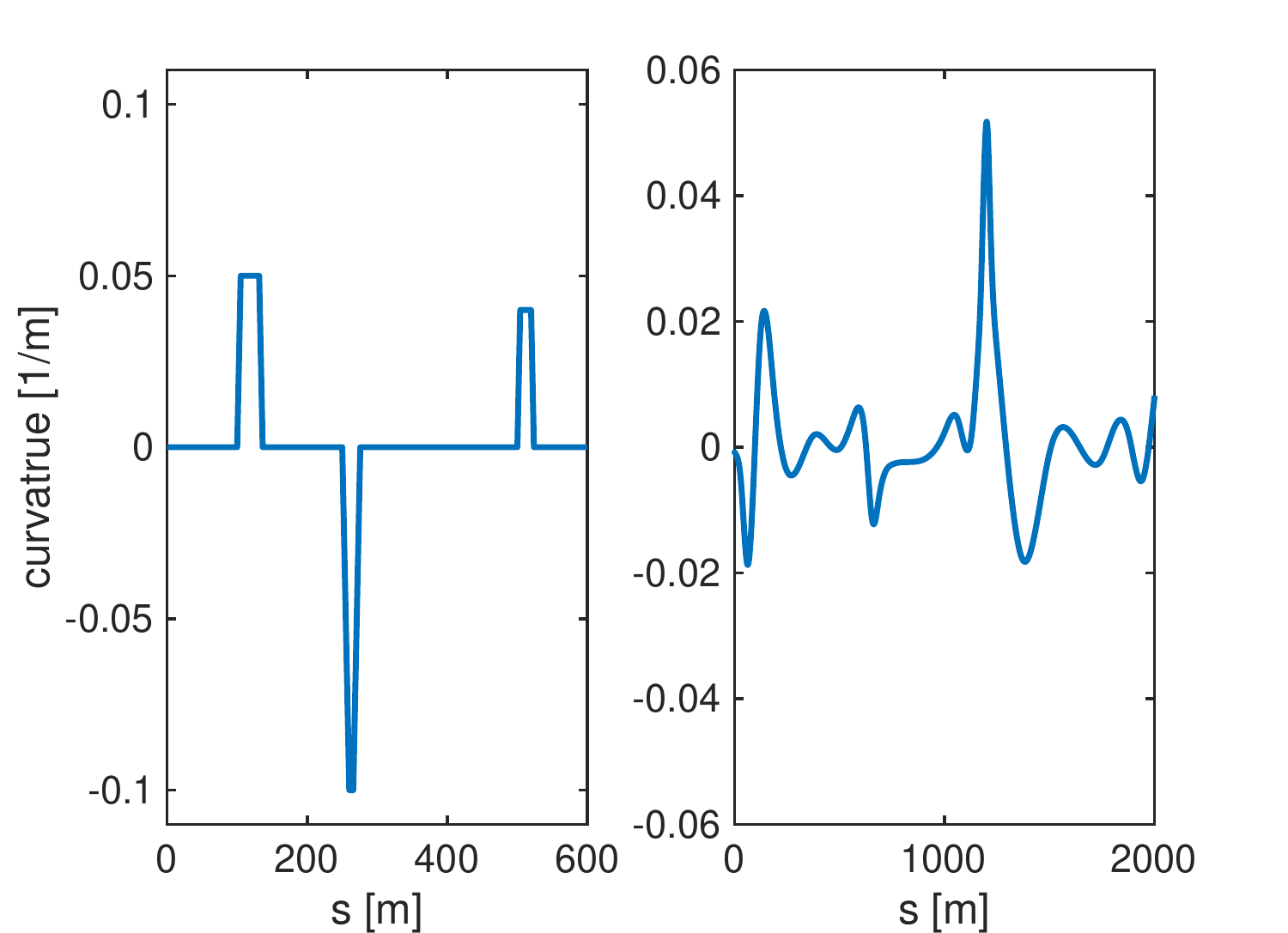}
\caption{Test road curvature, left the city road and right the country road.} \label{fig:test_curvature}
\end{figure}

\subsection{Setup}
To make our simulation study more realistic we use a bicycle model with linear tire forces for simulation \cite{rajamani2011}. To deal with low velocities where the model has a singularity, we use the method proposed in \cite{kabzan2019amz}. The parameters for the simulation model are are taken from \cite{rajamani2011} and \cite{Commonroad}.

To validate our proposed terminal constraints, based on the Discriminating Domain (DD), and the Neural Network Discriminating Kernel (NNDK), we test five different terminal constraints. Two baselines: no terminal constraint, and a zero terminal velocity constraint, as well as three based on our results: first, DD with fixed $\km$, second, DD with an adaptive $\km$ based on an algorithm introduced in Section \ref{sec:ada_km}, and finally NNDK with an adaptive $\km$. We also tested long horizons where the car can stop within the horizon (9 - 14s), and short horizons of 2s where only our proposed safe sets with the adaptive $\km$ approaches produce useful results.

Our experiments are split in two parts: first, we compare to the baselines, using long horizons, since the baselines only work with sufficiently long horizons, and in a second step we compare the DD terminal constraint with the NNDK terminal constraint using short horizons. 

\subsection{Adaptive Worst Case Curvature Computation} \label{sec:ada_km}
To implement our proposed terminal constraints, we need to know the worst curvature $\km$ of the road ahead. We propose to compute $\km$ only for a look ahead that is long enough. A long enough look ahead is one that allows us to react, to whatever the curvature in the not yet seen part of the road may be. Therefore, we first compute how long it would take to decelerate to zero from the end of the current horizon, and then compute $\km$ for 1.5 times the distance it would take to stop. Finally, we smooth $\km$ such that the controller is not disturbed, see Algorithm \ref{algo:innerApproxDisc} for the full method.

\begin{algorithm}[h!] \label{algo:innerApproxDisc}
	at time step $t$, given ${\km}_{,t-1}$\;
	get $v_T$, and $s_T$ from previous solution\;
	\Compute{}{
    $t_\text{stop} = v_T/\am$\;
	$s_\text{stop} = 0.5 \am t_\text{stop}^2 + v_T t_\text{stop}$\;
	$\km =  \max_{s \in [s_T,s_T+1.5 s_\text{stop}]} |\kappa(s)|$\;
	}
  \caption{Adaptive $\km$ computation}
\end{algorithm}
\vspace{-0.4cm}
\subsection{Implementation Details}
We implemented the MPC in Julia with the optimization framework JuMP \cite{jump} and use IPOPT to solve the resulting NLP \cite{ipopt}. Note that this implementation is not optimized for real-time use and the computation times are only reported to show relative differences. We refer to \cite{kouzoupis2018} for a discussion on real-time nonlinear MPC methods. Finally, we discretized the continuous dynamics \eqref{eq:dyn_mpc} using a trapezoidal integrator, and used a sampling time of 0.05s. 

\subsection{Results}
\subsubsection{Comparison to Baselines}
For our country road test we used a look ahead of 14s which corresponds to a horizon of 280 time steps, while for the city road test we used a look ahead of 9s which corresponds to a horizon of 180 time steps. Both are chosen such that the car can stop from the allowed top speed within the prediction horizon. We compare the long prediction horizon variants to our proposed approach that uses the DD-based terminal constraints and the adaptive $\km$ algorithm, with a look ahead of only 2s which corresponds to a horizon of 40 time steps. 

Figure \ref{fig:test_long}, shows that if a long horizon is used, only the MPC with a zero terminal velocity is significantly different, as the allowed speed limit is not reached. The MPC with the short horizon and the DD-based terminal constraint with adaptive $\km$, has again a different behavior and starts braking earlier than the other motion planners. However, the distance to the center line is very similar for all five variants, see lower part of Figure \ref{fig:test_long}.
\begin{figure}[h]
\centering
\includegraphics[width=0.725\linewidth]{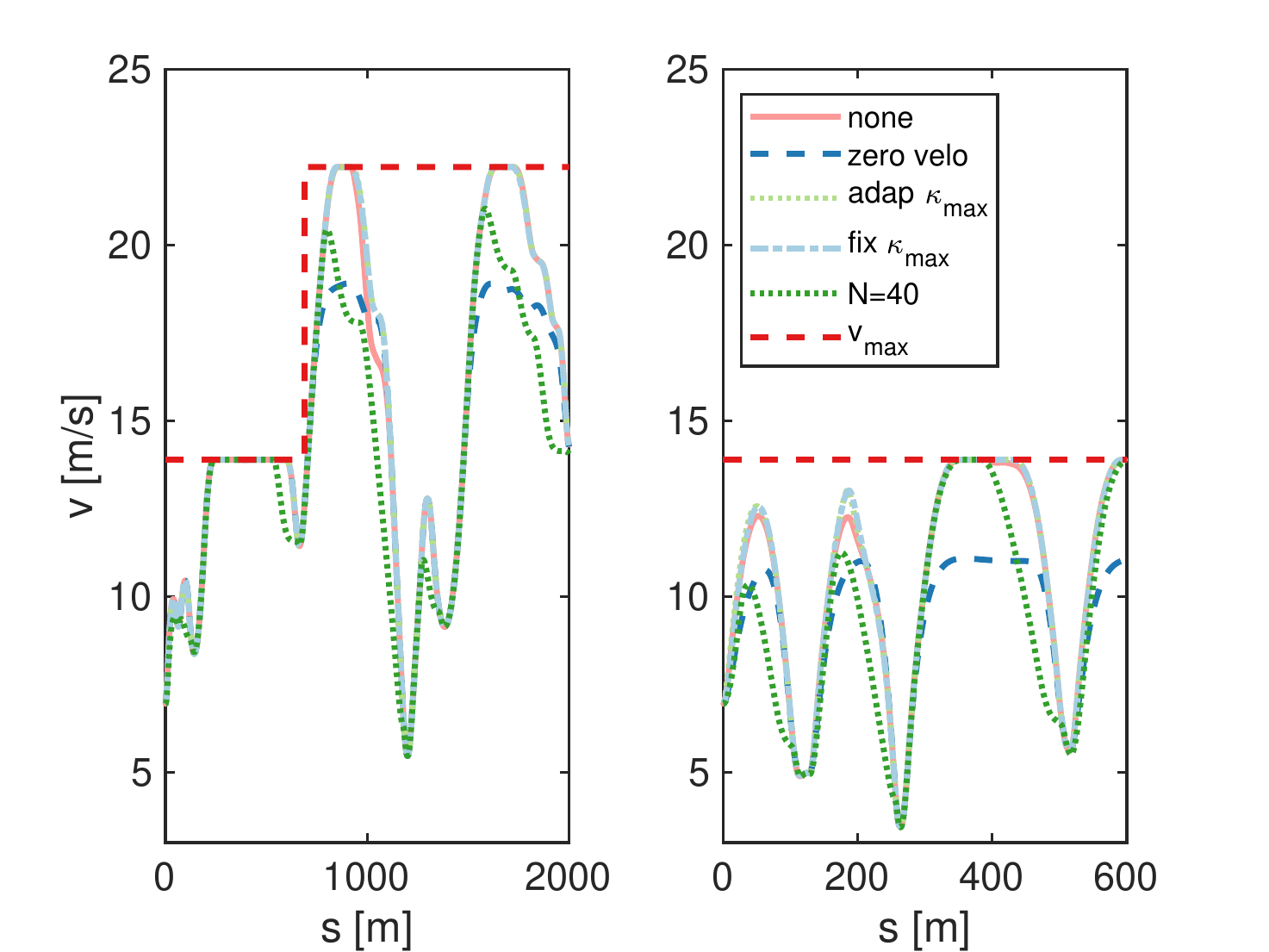}
\includegraphics[width=0.725\linewidth]{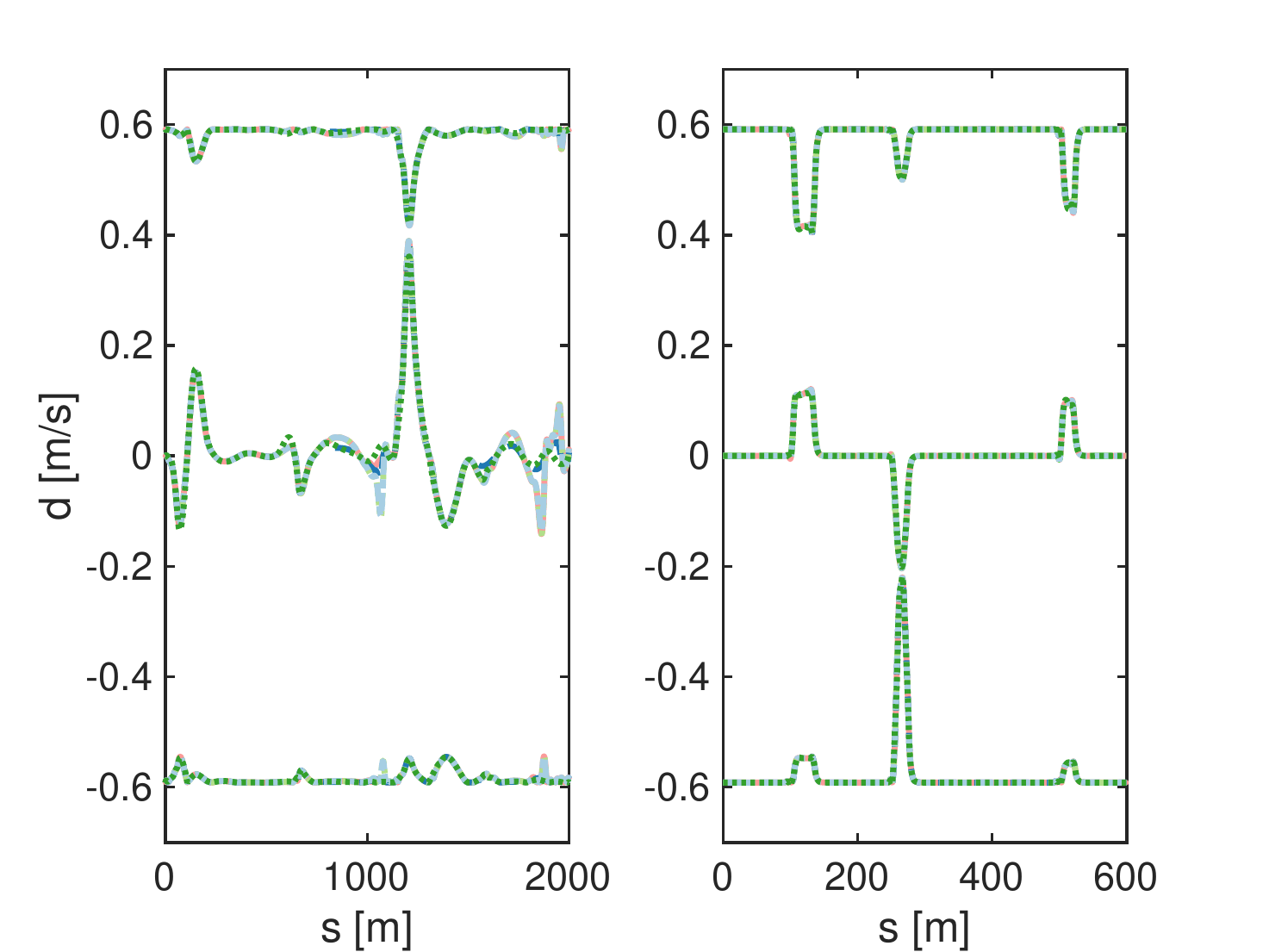}
\caption{Top: Speed profile for the comparison with the baselines. Bottom: $d$ profile for the comparison with the baselines, including a visualization of the heading dependent road constraints \eqref{eq:road_con}.} \label{fig:test_long}
\end{figure}

Similar patterns can also be seen in Table \ref{tab:test_long}, where the mean combined acceleration is the lowest for the MPC with the zero terminal velocity and the short horizon approach which is a consequence of the slower driving. For the computation time it is visible that the the horizon length has the largest impact, but for a fixed horizon length there is no clear pattern visible. The computation times are recorded on Desktop PC with an AMD 3900X processor running Ubuntu18.04.
\begin{table}[h]
\caption{Quantitative results, top country road, bottom city road}
\label{tab:test_long}
\centering 
\begin{tabular}{@{}l c c c c c@{}}\toprule
 & none & 0-velo& adap $\km$ & fix $\km$ & short\\
\midrule
mean time & 0.4273   & 0.4301  & 0.4251 & 0.4284  & 0.0428 \\
com acc& 1.1696  & 1.0640  & 1.1783 & 1.1774  & 0.9834  \\
\midrule
mean time & 0.3183   & 0.2242  & 0.3040 & 0.2659  & 0.0368 \\
comb acc& 1.0729  & 0.7850  & 1.1105 & 1.1066  & 0.8815 \\
\bottomrule
\end{tabular}
\end{table} 

\subsubsection{Discriminating Domain vs Discriminating Kernel}

Finally, we compare our path-following controller with the DD and the NNDK terminal constraint. Since we use a short horizon, we use the adaptive $\km$ approach to not be too conservative. Note that the MPC with no terminal constraint was not able to perform this task, and that imposing a zero terminal velocity or a non adaptive $\km$ results in the car driving very slowly. For example, the theoretical top speed for the zero terminal velocity approach would be 3.2m/s, and around 8m/s for the fixed $\km$ approach. As these results are of no practical interest we excluded them.
\begin{figure}[h]
\centering
\includegraphics[width=0.725\linewidth]{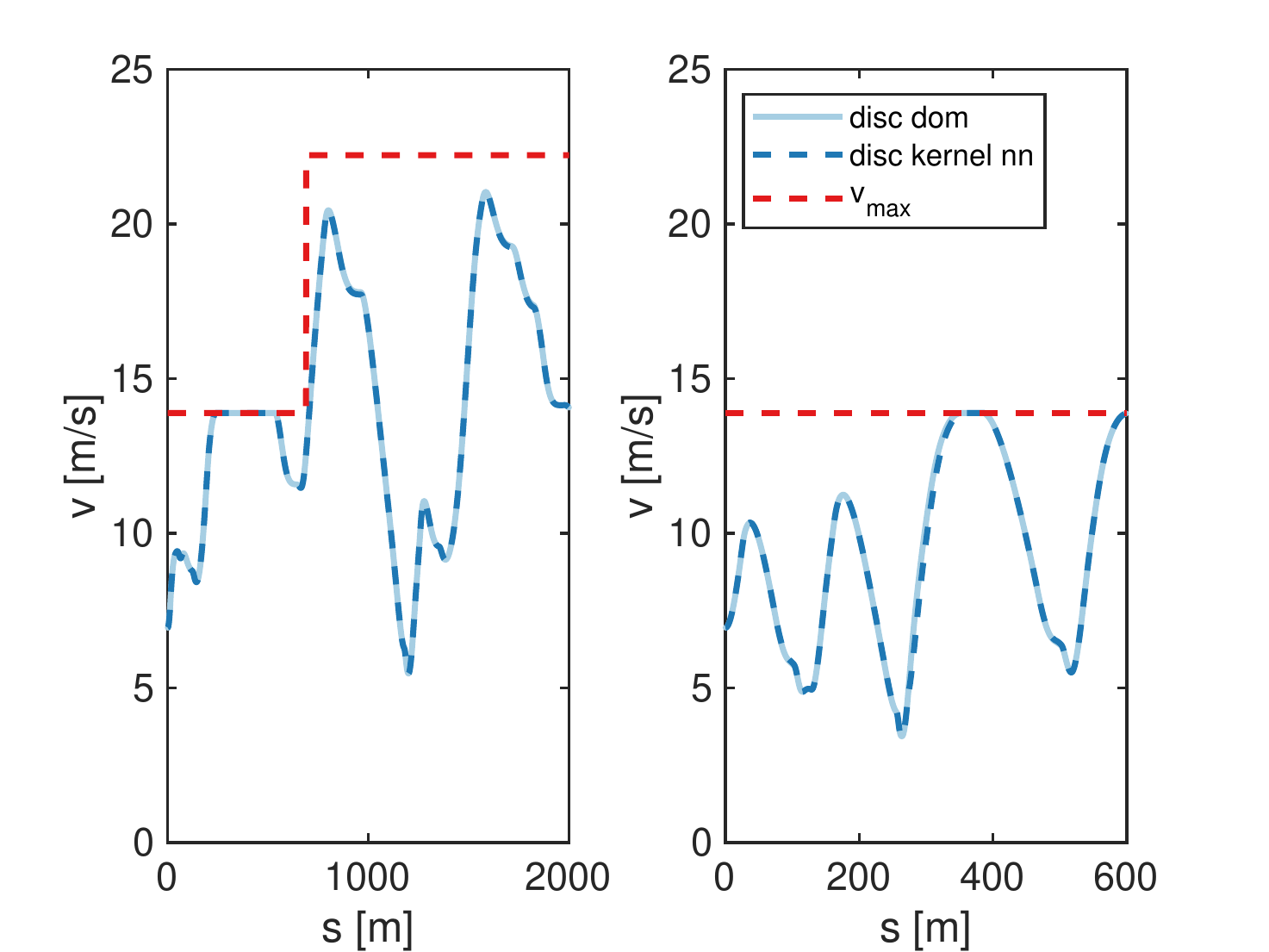}
\caption{Speed profile for the comparison between discriminating domain and discriminating kernel.} \label{fig:test_short}
\end{figure}

When comparing the velocity profiles in Figure \ref{fig:test_short} it is visible that there is basically no difference between the two approaches. The bigger difference can be seen in computation times that increase from 0.0428 to 0.0506s for the country road case and from 0.0368 to 0.0545s in the city road case.

Therefore, we conclude that for the here studied path-following MPC problem, using the DD-based terminal constraint is clearly a better solution. Although the NNDK-based terminal constraint performs the same in our simulation study, it has three main disadvantages: higher computation time, the loss of formal guarantees, and the loss in flexibility. However, if the problem would be changed, for example by using a learned car model \cite{kabzan2019learning,williams2017information}, the NNDK approach is still applicable, but the DD approach not.

\section{Conclusion and Future Work} 
\label{sec:con}

In this paper we introduce a safe optimization-based path-following controller for autonomous driving. The problem uses a kinematic model in curvilinear coordinates and has road constraints that include the orientation of the car as well as comfort constraints. Based on this path-following controller we derived a game theoretic version, where the road ahead is the adversary. Using this model we used tools from viability theory to establish safe sets for the original controller. The first is an analytical discriminating domain, that can be shown to also work if steering rate constraints are considered. The second approach computes the discriminating kernel using the gridding based discriminating kernel algorithm, to get a set that can be used in an MPC, we are using a neural network to approximate the discriminating kernel. Finally, we tested the terminal sets in simulation, and showed that if an adaptive $\km$ algorithm is used, our terminal sets allow to run short horizons, with only slightly reduced performance.

In future work, we want to include obstacle avoidance and interactions with other cars by using ideas similar to the one proposed in \cite{rss}. Additionally, it would be interesting to formulate the discriminating kernel algorithm directly using neural networks.

\bibliographystyle{plainnat}
\bibliography{references}

\end{document}